\documentclass[letterpaper, 10 pt, conference]{ieeeconf}  

\IEEEoverridecommandlockouts                              
\usepackage{graphicx} 
\usepackage{caption}
\usepackage{subcaption}

\usepackage{amsmath} 
\usepackage{amssymb}  

\title{\LARGE \bf
Optimistic Online Non-stochastic Control via FTRL}

\author{Naram Mhaisen and George Iosifidis
\thanks{N. Mhaisen and G. Iosifidis are with \textit{the Faculty of Electrical Engineering, Mathematics and Computer Science. Delft University of Technology. The Netherlands}.
        {\tt\footnotesize \{n.mhaisen, g.iosifidis\}@tudelft.nl}.}%
}

\newtheorem{theorem}{\textbf{Theorem}}
\newtheorem{lemma}{\textbf{Lemma}}
\newtheorem{assumption}{\textbf{Assumption}}

\newcommand{\dtp}[2]{\langle {#1}, {#2} \rangle}
\newcommand{\grd}{\nabla}
\renewcommand{\vec}[1]{\boldsymbol{#1}}

\DeclareMathOperator*{\argmin}{argmin}
\newcommand{\ti}{\tilde} 
\newcommand{\sumT}{\sum_{t=1}^T} 
\usepackage{nicefrac}
\usepackage{algorithm}
\usepackage{algorithmic}

\usepackage{dsfont}
\usepackage{tabularx}
\usepackage{makecell}
\usepackage{mathtools}
\mathtoolsset{showonlyrefs}
\begin{document}

\maketitle
\thispagestyle{empty}
\pagestyle{empty}

\begin{abstract}

This paper brings the concept of ``optimism" to the new and promising framework of online Non-stochastic Control (NSC). Namely, we study how NSC can benefit from a prediction oracle of unknown quality responsible for forecasting future costs. The posed problem is first reduced to an optimistic learning with delayed feedback problem, which is handled through the Optimistic Follow the Regularized Leader (OFTRL) algorithmic family. This reduction enables the design of \texttt{OptFTRL-C}, the first Disturbance Action Controller (DAC) with optimistic policy regret bounds. These new bounds are commensurate with the oracle's accuracy, ranging from $\mathcal{O}(1)$ for perfect predictions to the order-optimal $\mathcal{O}(\sqrt{T})$ even when all predictions fail. By addressing the challenge of incorporating untrusted predictions into online control, this work contributes to the advancement of the NSC framework and paves the way toward effective and robust learning-based controllers.

\end{abstract}

\section{Introduction}
We study the NSC framework originally introduced in \cite{pmlr-v97-agarwal19c}: Consider a discrete-time dynamical system where at each time slot, the controller observes the system state $\vec{x}_t \in \mathbb{R}^{d_x}$ and decides an action $\vec{u}_t \in \mathbb{R}^{d_u}$ which then induces a stage cost $c_t(\vec{x}_t, \vec{u}_t)$, and causes a transition to a new state $\vec{x}_{t+1}$. Note that the cost and the new state are revealed to the controller \emph{after} it commits its action. Similar to \cite{pmlr-v97-agarwal19c}, we study Linear Time Invariant (LTI) systems where the state transition is parameterized by matrices $A\in \mathbb{R}^{d_x\times d_x}$, $B\in \mathbb{R}^{d_x\times d_u}$ and a \emph{disturbance} vector $\vec w_t \in \mathbb{R}^{d_x}$:
\begin{align}
        \label{eq:LTIdynamics}
    \vec{x}_{t+1} = A \vec{x}_{t} + B \vec{u}_{t} + \vec{w}_{t}.
\end{align}
$\vec{w}_t$ can be arbitrarily set by an \emph{adversary}, and we only restrict it to be bounded, i.e., $\| \vec{w}_t\| \leq w, \forall t$. Similarly, the adversary is allowed to select any $l$-Lipschitz convex cost function $c_t:\mathbb{R}^{d_x}\times \mathbb{R}^{d_u}\mapsto \mathbb{R}$ at each slot. Since neither the cost nor disturbances are confined to follow a fixed and/or known distribution, the control problem is ``Non-stochastic".

The controller (or learner) aims to find a (possibly time-varying) policy that maps states to actions, $\pi: \vec{x}\mapsto \vec{u}$, from a given policy class $\Pi$, such that the stage costs $\{c_t(\vec{x}_t, \vec{u}_t)\}_{t=1}^T$ are as small as possible. To quantify the learner's performance, we employ the \emph{policy regret} metric, as presented in \cite{hazan20a-us}. Intuitively, the policy regret measures the accumulated differences between the cost incurred by the learner's policy, and that of a stationary unknown cost-minimizing policy designed with access to all future cost functions and disturbances: 
\begin{align}
    \mathcal{R}_T \doteq \sumT c_t\left(\vec x_t, \vec u_t\right) 
    - \min_{\pi\in\Pi} \sumT
    c_t\big(\vec x_t(\pi), \vec u_t(\pi)\big),
    \label{eq:policy-regret}
\end{align}
where $(\vec x_t(\pi), \vec u_t(\pi))$ is the \emph{counterfactual} state-action sequence that would have emerged under the benchmark policy, whereas $(\vec x_t, \vec u_t)$ is the \emph{actual} state-action sequence that emerged from following possibly different policies by the learner. A sub-linear regret means that the cost endured by the learner will converge to that of the optimal unknown policy at the same sub-linear rate, i.e., $\nicefrac{\mathcal{R}_T}{T}\rightarrow 0$ as $T\rightarrow \infty$. Note that the benchmark depends on the actual \emph{witnessed} costs and disturbances, encoding a stronger concept of ``adaptability" unlike the pessimistic $\mathcal{H}_\infty$ controllers, which assume worst case costs \cite{karapetyan2022regret}. Remarkably, the Gradient Perturbation Controller (GPC) achieves $\mathcal{O}(\sqrt{T})$ regret, which is order optimal, and was shown to deliver superior performance to the more conventional $\mathcal{H}_2$ controllers \cite{pmlr-v97-agarwal19c}.

In contrast to the typical NSC framework, here we consider the existence of a \emph{prediction oracle}. This oracle provides the learner, before committing to the action $\vec{u}_t$, with a forecast for the as-yet-unobserved cost functions $\{c_\tau(\cdot, \cdot)\}_{\tau=t}^{t+d}$ for a specific horizon of $d$ slots. We denote the oracle's outcome as $\tilde{c}_\tau(\cdot,\cdot)$. Notably, the predictions themselves can be influenced by the adversary, meaning that no assumptions on the oracle's \emph{accuracy} are made; the forecasted parameters may either deviate arbitrarily from the truth, or be actually accurate in the best case.

\subsection{Motivation}
The motivation for this addition on the NSC framework stems from the abundance of machine learning forecasting models, which provide high potential improvement if they are adequately accurate.
The effect of such predictions on the classical regret metric (not the policy regret) has been studied in the literature of optimistic online learning, where the ultimate objective is to provide regret guarantees that scale with the accuracy of predictions while always staying sub-linear. Namely: $\mathcal{R}_T =\mathcal{O}(1)$ when all predictions are accurate and $\mathcal{O}(\sqrt{T})$ in all cases. That is, we are assured of reaching the performance of the benchmark policy, yet at a significantly improved rate when predictions happen to be precise. Hence, optimistic online learning represents a highly desirable combination of the best of both worlds: optimal worst-case guarantees with achievable best case guarantees. In fact, optimistic learning algorithms have been attracting considerable attention as the driving force behind recent state-of-the-art results in online constrained optimization \cite{10021291}, online discrete optimization \cite{mhaisen2022optimistic}, online sub-modular optimization \cite{si2023online}, and online fairness \cite{aslan2024fair} to name a few.

Unfortunately, such optimistic algorithms are yet to find their way to online NSC. This might be surprising given that previous NSC results were obtained through a streamlined reduction to the standard Online Convex Optimization (OCO) framework \cite{hazan-oco-book}. Hence, one might anticipate that optimistic NSC algorithms can be derived using a similar approach. Interestingly, this is not the case. Combining optimistic learning and NSC poses unique challenges for both frameworks. \textbf{First}, existing optimistic learning algorithms do not consider cost functions \emph{with memory} and hence cannot handle states. Particularly, these algorithms update their regularizers at each time slot based on the accuracy of the prediction used by the preceding action  \cite{rakhlin2013online,mohri-aistats2016}. However, in stateful systems, an action made at $t$ will have an effect that spans across all slots until $t\!+\!d$, and thus uses predictions for all these slots. Since the accuracy of such multi-step predictions is not available until  $t\!+d+\!1$, we cannot update the regularizer in the standard way. \textbf{Second}, the guarantee of NSC is established via a reduction to the OCO with Memory (OCO-M) framework \cite{NIPS2015_38913e1d}, which in turn is reduced to the standard OCO \cite{MAL-018} via the concept of slowly moving decision variables \cite[Thm. 4.6]{pmlr-v97-agarwal19c} \cite[Thm. 3.1]{NIPS2015_38913e1d}. This later reduction cannot be utilized in optimistic learning where accurate predictions lead to little or \emph{no} regularization, driving consecutive decisions of the optimistic algorithm to vary arbitrarily (up to the set diameter) \cite[Sec $2.2$]{mohri-aistats2016}, \cite[Sec $7.4$]{pmlr-v76-joulani17a}.

This paper tackles exactly these challenges and aims to answer the question: \textit{is it possible to design an online learning algorithm whose policy regret is commensurate with the accuracy of an exogenous prediction oracle, while always staying sub-linear?} We answer this positively and builds upon recent advances in online learning, introducing, to our knowledge, the first optimistic controller for NSC.

\subsection{Contributions}
We achieve such optimistic guarantees by departing from the standard analysis approach of reducing the learner's non-stationary policy to a stationary one \cite{pmlr-v97-agarwal19c, zhao2022non, 10384153}, and instead directly analyzing the non-stationary policy. Specifically, to address the first challenge, we demonstrate that the additive separability of the linearized costs allows expressing the costs as a sum of \emph{memoryless} but \emph{delayed} functions of each of the decision variables. Next, to tackle the second challenge, we analyze the performance of each decision variable \emph{separately} via an alternative reduction to the framework of ``optimism with delay" \cite{optimdelay21}. Nonetheless, we customize this later framework with a specific ``hint" design that exploits the structure of NSC where the cost is indeed delayed but still \emph{gradually} being revealed at each step, leading to tighter bounds. We make these intuition-focused points concrete in the upcoming analysis. 

The main contribution is an optimistic controller with policy regret scaling from $\mathcal{O}(1)$ to $\mathcal{O}(\sqrt{T})$, based on prediction accuracy. The methodology hinges on a new perspective on stateful systems (with a prediction oracle) as systems with delayed feedback, for which we build upon recent results on delay and optimism. The next section reviews the related works. Sec. \ref{sec:prelem} provides the necessary background for the new algorithm, \texttt{OptFTRL-C}, introduced in Section \ref{sec:algo}. We then present numerical examples in Sec. \ref{sec:ne} and conclude.

\section{Related Work}
The NSC problem was initiated in the seminal work of \cite{pmlr-v97-agarwal19c}, which introduced the first controller with sub-linear policy regret for dynamical systems, generalizing the classical control problem to  adversarial convex costs functions and adversarial disturbances. These results were further refined for \emph{strongly} convex functions \cite{simchowitz2020making,Agarwal-log,foster20b-log2}; and systems with fixed or adversarially-changing constraints on the actions \cite{Li_Das_Li_2021, liu2023nsc-const}. Follow-up works also looked at the NSC problem under more general assumptions on the system matrices,  including unknown $(A,B)$ \cite{hazan20a-us}, systems with bandit feedback \cite{GHH}, and time-varying systems \cite{gradu23a-ltv}. Expectedly, the regret bounds worsen in these scenarios, such as increasing to $\mathcal{O}(T^{\nicefrac{2}{3}})$ for unknown systems and $\mathcal{O}(T^{\nicefrac{3}{4}})$ for systems with bandit feedback. Efficiency is also investigated in \cite{10384153} with projection-free methods. Going beyond the typical bounds in terms of $T$, \cite{mhaisen2023adaptive} introduced an adaptive FTRL based controller whose bound is proportional to the witnessed costs and perturbations: $\mathcal{O}((\sumT g_t^{2})^{\nicefrac{1}{2}})$ instead of $\mathcal{O}({\sqrt{T}})$, where $g_t$ depends on the witnessed costs and perturbations. These works do not consider the existence of an untrusted prediction oracle, as the model adopted here. 

The NSC framework was also investigated using other metrics such as dynamic regret \cite{zhao2022non}, adaptive regret \cite{zhang2022adversarial}, and competitive ratio \cite{shi2020online, goel2022competitive}.
For dynamic and adaptive regret, methods with static regret guarantees, as discussed here, are used as building blocks for ``meta" algorithms with static/adaptive guarantees \cite{gradu23a-ltv, simchowitz20a-part-state}. For competitive ratio, it was demonstrated in \cite{goel2023best} that a regret guarantee against the optimal DAC policy automatically implies a competitive ratio with an additive sub-linear term. Hence, the presented algorithm is still highly relevant even for other metrics.

Our work is also related to the robust MPC literature \cite{bemporad2007robust, dullerud2013course}, where (possibly inaccurate) predictions are used. However, the main difference is that while our algorithm is robust to bad predictions, it is not designed based on them. Specifically, our benchmark policy changes when worst case costs and predictions are not witnessed (as dictated by the regret metric). A newer line of research studies the MPC-style algorithms under the regret metric \cite{lin2021perturbation, lin2022bounded}. Perfect predictions are assumed in \cite{lin2021perturbation}, whose bound was later generalized in \cite{lin2022bounded} with parameterized predictions error. While these works consider the refined dynamic regret metric, their regret scales linearly with the prediction error.

Optimism extends the OCO framework by incorporating predictions of future costs, akin to ``certainty equivalent" control. However, a key focus of optimistic algorithms is ensuring that the regret bound does not depend linearly on the accuracy of these predictions. The optimistic learning framework, in its current prevalent form, originated in \cite{rakhlin2013online} where an Online Mirror Descent (OMD) algorithm was presented. Thereafter, optimism was studied under another equivalent framework (FTRL). Optimistic FTRL (OFTRL) has since undergone multiple improvements \cite{pmlr-v76-joulani17a, naram-jrnl}
and we refer the reader to \cite[Sec. 7.12]{orabona2021modern} for comprehensive overview. Only recently, a close connection between delay and optimism emerged \cite{optimdelay21}, a development we leverage in our current analysis of the NSC framework.

While optimism has been studied for stochastic predictions \cite{chen2015online, chen2016using}, and adversarial predictions \cite{mohri-aistats2016, mhaisen2022optimistic}, these findings have not been applied to dynamical systems. In dynamical systems, the study of predictions is limited to either \emph{perfect} predictions \cite{power-prediction, 10.5555/3454287.3455620}, or \emph{fixed} quadratic (thus, strongly convex) cost functions  \cite{zhang2021regret, yu2022competitive, li_robustness_2022}. Our paper contributes to filling this gap. Predictions may also be viewed via the lens of ``context" \cite{levy2023optimism} in the problem of stochastic MDPs with \emph{finite} states and actions. Lastly, the previous works of \cite{li_robustness_2022, yu2022competitive} consider that the full predictions are provided 
\emph{a priori}, meaning that the learner cannot benefit from updated predictions. Like MPC, we drop this assumption and allow the learner to use the most recent predictions. 

\section{Preliminaries}
\label{sec:prelem}
\textbf{Notation.} We denote scalars by small letters, vectors by bold small letters, and matrices by capital letters. Time indexing is done via a subscript. We denote by $\{\vec{a}_t\}_{t=1}^T$ the set $\{\vec{a}_1, \dots, \vec{a}_T\}$. $M\! =\! [M^{[i]}| M^{[j]}]$ denotes the augmentation of $M^{[i]}$ and $M^{[j]}$.  $\|\cdot\|$ denotes the $\ell_2$ norm for vectors and the Frobenius norm for matrices.  $\|\cdot\|_*$ is the dual norm.$\dtp{\cdot}{\cdot}$ is the dot product for vectors and the Frobenius product for matrices.
$\|\cdot\|_{\text{op}}$ is the matrix spectral norm (the induced $\ell_2$ norm). We use $h_{a:b}$ to indicate $\sum_{s=a}^b h_s$ when $s$ is irrelevant, and $f(\cdot)$ when the function's argument are irrelevant.

\textbf{DAC policy class.} The class $\Pi$ under consideration in this paper, and in the broader NSC literature, is the Disturbance Action Controllers (DAC) (see \cite[Ch. 6]{hazan-nsc-book} for a general reference). The motivation behind DAC lies in the combination of expressive power and efficient parametrization. Specifically, DAC can approximate the large class of linear controllers, which, for instance, is guaranteed to include the universally optimal controller in the LQR settings, and the linear positive systems with linear objective function \cite{Rantzer}. Simultaneously, the use of DAC actions has been demonstrated to induce \emph{convex} cost functions in its parametrization.

A policy $\pi \in \Pi^{}$, with a memory length of $p$, is parameterized by $p$ matrices $M\doteq\big[M^{[1]}|M^{[2]}|\dots |M^{[j]}|\dots|M^{[p]}\big]$, $M^{[j]}\in \mathbb{R}^{d_u \times d_x}$, and a fixed stabilizing controller $K$. We define the set $\mathcal{M}\doteq\{M:\|M\| \leq \kappa_M \}$, with the bounded variable assumption. The action at a step $t$ according to a policy $\pi_t\in \Pi$ is then calculated as follows:
\begin{align}
    \label{eq:dac-action}
    \vec{u}_{t} = K \vec{x}_t + \sum_{j=1}^p M^{[j]}_t \vec{w}_{t-j}.
\end{align}
$K$ is pre-calculated and provided as an input. This is formalized through the \emph{strong stability} assumption \cite[Def. 3.1]{cohen2018online}, which is standard in NSC. It ensures the existence a controller $K$ such that $\|(A+BK)^t\|_{\text{op}}\leq\kappa (1-\delta)^t$ for $\delta\in(0,1]$, where $\kappa>0$. Here, we make a slightly stricter assumption, enabling us to simplify the analysis:
\begin{assumption}
The system $(A,B)$ is intrinsically stable: $\|A\|_{\text{op}}\leq 1-\delta$ for some $\delta\in(0,1]$.
\end{assumption}

The above assumption allows us to satisfy the strong stability assumption with $K$ being the zero matrix. Otherwise, we can revert to the strong stability assumption itself. This simplification facilitates the analysis without much loss in generality, as discussed in, for example, \cite[Remark 4.1]{zhang2022adversarial}. Additionally, we assume $\|B\|\leq \kappa_B$. The boundedness of $\|A\|$ follows from its spectral norm bound.

When using DAC policies, it is known that the state at $t+1$ can be described as a linear transformation of the parameters chosen by the learner in the previous $t$ slots $M_1,M_2,\dots,M_t$:
\begin{align}
    \vec x_{t+1} &= \sum _{i=0}^t A^i \bigg(\!B\sum_{j=1}^p \big(M^{[j]}_{t-i}\ \vec w_{t-i-j}\big) + \vec w_{t-i}\!\bigg)
    \\
    & =  \sum _{i=0}^{t} A^i \big(BM_{t-i}\ \vec {\overline{w}}_{t-i-1} + \vec w_{t-i}\big), 
    \label{eq:state-nonstat-m}
\end{align}  
where we defined $\vec {\overline{w}}_{t-i-1}\doteq (\vec w_{t-i-1}, \dots, \vec w_{t-i-p})$ so as to express the vector $\sum_{j=1}^p M_t^{[j]} \vec w_{t-j}$ compactly as $M_t\overline{\vec w}_{t-1}$. 
The above expression for the state can be obtained by simply unrolling the dynamic in \eqref{eq:LTIdynamics} and assuming, w.l.o.g that the initial state $\vec x_1$ (before executing $M_1,\dots,M_t$) is $\vec 0$. This is proven, e.g., in \cite[Lem. 7.3]{hazan-nsc-book} and \cite[Lem. 4.3]{zhang2022adversarial}. 

\textbf{Cost functions}. While our guarantees still hold for general convex cost functions, we focus here on the linear case:
\begin{assumption}
The cost is linear in the state and action $ c_{t}(\vec{x}_t, \vec{u}_t) = \dtp{\vec \alpha_t}{\vec{x}_t} + \dtp{\vec \beta_t}{\vec{u_t}}.$ $\|\vec \alpha_t\| \leq \alpha$ and $\|\vec \beta\| \leq \beta$.
\end{assumption}
The linearity assumption provides a useful structure in the analysis (separability) and enables us to quantify the prediction error in terms of the parameters $\vec \alpha_t$ and $\vec \beta_t$. It does not, however, compromise the presented regret guarantees. In fact, the linear costs are the most challenging\footnote{As indicated in Sec. II, in the case of strongly convex costs, a tighter regret bound of $\mathcal{O}(\log(T))$ (compared to $\mathcal{O}(\sqrt{T})$) is possible \cite{foster20b-log2}.} in the online learning settings; the regret caused by a sequence of general convex function can be indeed upper bounded by the regret caused by linearization of those functions. Thus, online learning works often focused on the linear case (see discussion on linearization in \cite[Sec. 2.4]{MAL-018}). Future costs can thus be predicted through the oracle's output $\{\vec\alpha_\tau, \vec \beta_\tau \}_{\tau=t}^{t+d}$. 

\textbf{Predictions.} The prediction model we use has several advantages over those that appear in the preceding section. $(i)$ We allow the prediction oracle to update its forecast \emph{at every decision slot $t$ } (similar to MPC). This flexibility is important, as in practice predictions can improve with time. $(ii)$ the analysis reveals that the parameter $d$ needs to scale logarithmically with $T$. This implies that predictions are not required for the entire future. $(iii)$ the presented algorithm and its guarantees place no assumptions on the predictions' quality. To our knowledge, this represents the most general prediction-based setting for online control.

\section{Online control with Optimistic FTRL}
\label{sec:algo}
We introduce first few definitions to facilitate the presentation of the algorithm. Define the \emph{forward cost} function:
\begin{align}
    F_{t}(M) \doteq \sum_{i=0}^{d} f^{(i)}_{t+i}(M) \label{eq:forward_cost},
\end{align}
where each $f^{(i)}_{t+i}(M)$ function shall describe the contribution of $M$ to the cost experienced at slot $t+i$, and is defined as
\begin{align}
f_t^{(i)}(M) \doteq
\begin{cases}
  \dtp{\vec \alpha_t}{\vec \psi^{i-1}_t(M)} & \text{if } i \geq 1, \\
  \dtp{\vec \beta_t}{\vec \psi_t(M)} & \text{if } i = 0, 
\end{cases}
\label{eq:partial_functions}
\end{align}
with $\vec{\psi}^{i}_t: \mathbb{R}^{d_u\times (d_x p)} \mapsto \mathbb{R}^{d_x},$ and $\vec{\psi}_t: \mathbb{R}^{d_u\times (d_x p)} \mapsto \mathbb{R}^{d_u}$ being the following linear transformations, which are used to simplify the presentation of the action and state expression in \eqref{eq:dac-action} and \eqref{eq:state-nonstat-m}, respectively: 
\begin{align}
 \vec{\psi}^{i}_{t+1}(M)&\doteq  A^i(BM\overline{\vec{w}}_{t-i-1}+\vec{w}_{t-i}), \label{eq:state-trans}
 \\
 \vec{\psi}_t(M)&\doteq   M\overline{\vec{w}}_{t-1}.\label{eq:action-trans}
\end{align}
The role of the functions in \eqref{eq:partial_functions} will become clear later in the analysis. Roughly, the cost $c_t$ will be expressed as a sum of them. 
Denoting by $G_{t}^{(i)}= \grd_{M} f^{(i)}_t(M)$, we have:
\begin{align}
G_{t}^{(i)} =
    \begin{cases}
    B^\top(A^{i-1})^\top \vec \alpha_t \overline{\vec{w}}_{t-i-2}^\top & \text{if } i \geq 1, \\
  \vec{\beta}_t \overline{\vec{w}}_{t-1}^\top  & \text{if } i = 0.
    \end{cases}
    \label{eq:grads}
\end{align}
Note that $G_{t}^{(i)}$ is a ${d_u \times d_xp}$ matrix with the $(m,n)$-th element being the partial derivative of $f_t^{(i)}$ w.r.t. the $(m,n)$-th element of $M$.
From the above, we can get the bounds
\begin{align}
\|G_{t}^{(i)}\| \leq
\begin{cases}
  \alpha\kappa_Bpw (1-\delta)^{i-1} &\doteq g^{(i)} \quad \text{if } i \geq 1,
  \\
  \beta p w &\doteq g^{(0)} \quad \text{if } i = 0, \label{eq:grad-mags}
\end{cases}
\end{align}
using that for any matrices $A$,$B$, and vector $\vec w$ $\|AB\|=\|A\|\|B\|$ and $\|A\vec w\| \leq \|A\|_{\text{op}} \|\vec w\|$.
We also define the prediction $\ti{G}_{t}^{(i)}$, which we can construct by plugging the oracle's output in \eqref{eq:grads}, and hence we have the \emph{partial} prediction error:
\begin{align}
    \label{eq:deltai}
    &\ \Delta_t^{(i)} \doteq \|G_{t}^{(i)} - \ti{ G}_{t}^{(i)} \|
\end{align}
We highlight that the prediction error magnitude decreases exponentially with $i$: $\Delta_t^{(i)}\!\!\propto\! (1-\delta)^i \leq e^{-i}$. This follows from \eqref{eq:grad-mags} and the fact that predicting $\ti{G}_{t}^{(i)}, i\geq1$ amounts to predicting $\vec \alpha_t$ and attenuating it by a multiplication with $A^{i-1}$. We refer to $i$ therefore as the attenuation level.

Similarly, we define ${G}_{t}$ as the matrix of partial derivatives of $F_t(\cdot)$, and ${\ti G}_{t}$ as its prediction.
Lastly, we define the hybrid hint matrix $H_t$, which aims to approximate the sum $G_{t-d:t}$
\begin{equation}
    \label{eq:hint-matrix}
    \!\!\!\!H_t\!\doteq\!\underbrace{\!\sum _{i=0}^{d-1}\! \bigg(\sum_{j=0}^{\substack{d-i-1}}\!\!G_{t-d+i+j}^{(j)}}_{\text{observed at $t$}} 
    \ \!\!+\!\!\!\!
    \underbrace{\sum_{j'=d-i}^d\!\!\! \ti{G}_{t-d+i+j'}^{(j')}\!\bigg) \mathds{1}_{d\geq1}\!+\! \ti{G}_{t}}_{\text{future predictions}} ,
\end{equation}
and we denote by $\Delta_t$ the prediction error  $\Delta_t\doteq\|G_{t-d:t} - H_{t}\|$. Due to the definition of $F_t(\cdot)$, certain elements of the summands (in $G_{t-d:t}$) are partially observed at $t$ and are hence directly used in constructing $H_t$. The remaining elements in the sum are obtained from the prediction oracle.

\begin{figure*}
      \centering
      \includegraphics[width=0.85\textwidth]{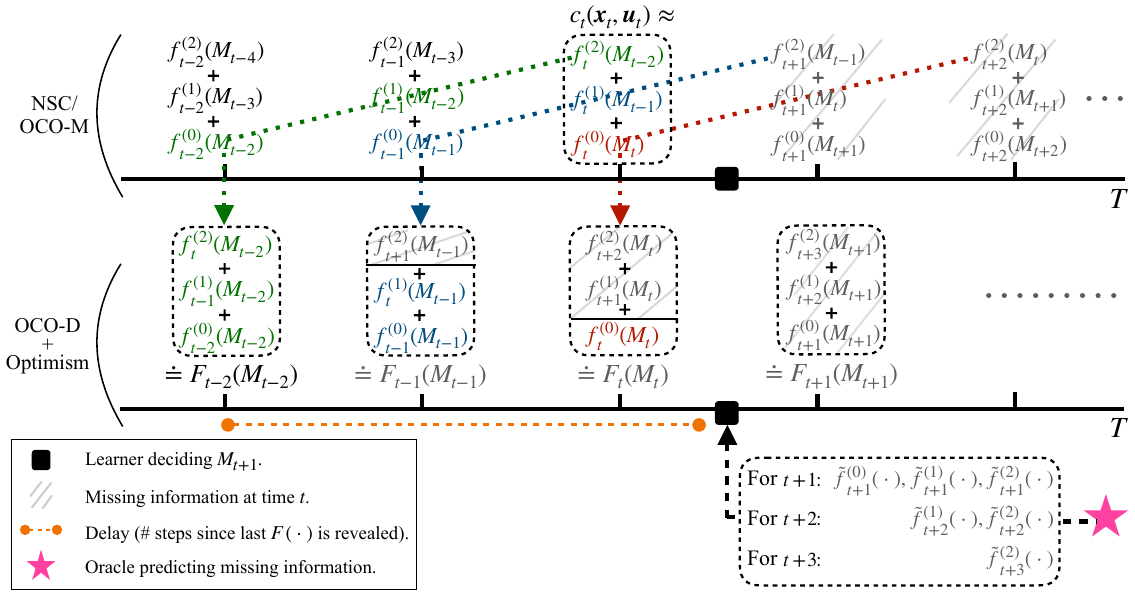}
      \caption{The methodology of designing \texttt{OptFTRL-C}. Up: The NSC to OCO-M reduction, with parameter $d=2$ (Sec. \ref{subsec:A}). Down: an equivalent \emph{delayed OCO} formulation, obtained via rearrangement, which we append with an oracle (Sec. \ref{subsec:B}).}
    \label{fig:method}
    \vspace{-0.3cm}
\end{figure*}

With these definitions at hand, we can now introduce the main algorithmic step. We propose an algorithm (\texttt{OptFTRL-C}) for optimizing the policy parameters $M_t$, $t\in [T]$. The algorithm uses the update formula:
\begin{equation}
    M_{t+1} = \argmin_{M\in\mathcal{M}} \left\{\dtp{ G_{1:t-d}+ H_{t+1}} {M} +  r_{t+1}(M)\right\},
    \label{eq:update_step}
\end{equation}
where $r_t(\cdot)$ are strongly convex regularizers defined as:
\begin{align}
    &r_{t+1}(M) = \frac{\lambda_{t+1}}{2} \|M\|^2,
    \\
    &\lambda_{t+1} = \frac{4}{\kappa_M}\max_{j\leq t-d-1}\!\!\Delta_{j-d+1:j} + \frac{\sqrt{5}}{\kappa_M} \sqrt{\sum_{i=1}^{t-d} \Delta_{i}^2 } \label{eq:lambda-update}.
\end{align}
In essence, the regularizer is a $\lambda_{t+1}$-strongly convex function, where $\lambda_{t+1}$ is proportional to the observed prediction error up to $t$. At each $t$, we set the strong convexity to be the maximum sum of $d$ consecutive witnessed errors, added to the root of the accumulated squared witnessed errors. This later term is aligned with memoryless OFTRL, whereas the former is necessary to adjust for the memory/delay effect. 

\begin{algorithm}[t]
	\caption{\mbox{\hspace{-.05mm}\texttt{OptFTRL-C}}}
	\begin{algorithmic}[1]
		\REQUIRE System $(A,B)$, parameter $d$, DAC parameters $\kappa_M, p$. 
  		\ENSURE Actions $\vec{u}_t$ at each slot $t=1,\ldots, T$.
		
		\FOR {each time slot $t = 1, \dots, T$}
        \STATE Use action $\vec{u}_t=\sum_{j=1}^p M_t^{[j]}\ \vec{w}_{t-j}$
            \STATE Observe $c_t(\vec{x_t}, \vec{u_t})$ and record the gradient  $G_{t-d}$
            \STATE Observe new state $\vec{x_{t+1}}$ and record $\vec w_t$
            \STATE Calculate $\Delta_{t-d}$ and update  parameter $\lambda_{t+1}$ via \eqref{eq:lambda-update}
            \STATE Receive future predictions $\{c_{\tau}(\cdot, \cdot)\}_{\tau=t+1}^{t+d}$
            \STATE Construct $H_{t+1}$ as in \eqref{eq:hint-matrix}
            \STATE  Calculate ${M}_{t+1}$ via \eqref{eq:update_step}
		\ENDFOR
	\end{algorithmic}\label{alg:OptFTRL-C}
\end{algorithm}

The steps of the \texttt{OptFTRL-C} routine are outlined in Algorithm \ref{alg:OptFTRL-C}.
\texttt{OptFTRL-C} first executes an action $\vec{u}_t$ (line $2$). Then, the cost function is revealed, completing the necessary information to compute $G_{t-d}$ (line $3$, recall that all the costs from $t-d,\dots,t$ are required to know $G_{t-d}$). The system then transitions to state $\Vec{x}_{t+1}$, effectively revealing the disturbance vector $\Vec{w}_t$ (line $4$). At this point, we can calculate $\Delta_{t-d}$ (since we know the ground truth $G_{t-d}$) and update the strong convexity parameter accordingly (line $5$). Then, the oracle forecasts the next $d$ costs functions (line $6$), enabling us to construct the hybrid hint matrix (line $7$). Finally, the next action $\vec{u}_{t+1}$ is committed through updating the policy parameters (line $8$). The  regret of this routine is characterized in the following theorem:

\begin{theorem}
\textit{Let $(A,B)$ be an LTI system, and $\{c_t(\cdot,\cdot)\}_{t=1}^T$, $\{\vec{w}\}_{t=1}^T$ be any sequence of stage costs and disturbances, respectively. Let $\Delta^{(i)}_t$ be the prediction error at $t$ with attenuation level $i$, as defined in \eqref{eq:deltai}. Then, with memory parameter $d$ defined as in Lemma \ref{lem:approimation}, and under Assumptions $1$ and $2$, algorithm \texttt{OptFTRL-C} produces actions $\{\vec u_t\}_{t=1}^T$ such that for all $T$, the following holds:}
\label{thm:main}
\begin{equation}
    \mathcal{R}_T = \mathcal{O}\left(\!\sqrt{\sumT\!\bigg(\!\sum_{i=0}^d \sum_{j=i}^d \!\Delta^{\!(d-j+i)}_{t+i}\!\!\bigg)^{\!\!2}}\ \right). \notag
    \vspace{2mm}
\end{equation}
\end{theorem}
\textbf{Discussion.} \texttt{OptFTRL-C} achieves the sought-after accuracy-modulated bound that holds for systems with memory. It generalizes previous optimistic online learning bounds by incorporating memory (hence states), and generalizes previous online non-stochastic control bounds by handling predictions of unknown quality. Namely, \texttt{OptFTRL-C}'s bound has the following characteristics.

\underline{\textit{Prediction-commensurate}}: in the \emph{best} case predictions ($\Delta_{t}^{(i)} = 0$, $\forall t$), the bound collapses to $\mathcal{O}(1)$, which is \emph{constant}. On ther other hand, in the \emph{worst} case, we get
\begin{align}
    &\!\sum_{i=0}^d \sum_{j=i}^d \!\Delta^{\!(d-j+i)}_{t+i} \leq 2\!\sum_{i=0}^d \sum_{j=i}^d \! g^{(d-j+i)} 
    = 2 \sum_{i=0}^d \sum_{k=i}^d g^{(k)}= \notag
    \\
    &2\!\sum_{k=0}^d \sum_{i=0}^k\!g^{(k)}= 2\!\sum_{k=0}^d (k\!+\!1) g^{(k)} = 2g^{(0)}\!\!+\! 2\!\sum_{k=1}^d (k\!+\!1) g^{(k)}  \notag 
    \\
    &\leq 2\beta pw + \frac{2\alpha\kappa_Bpw}{\delta^2} \doteq m \notag
\end{align}
where the first equality follows from the triangular inequality and \eqref{eq:grad-mags}, and in the last inequality we used $\sum_{i=0}^\infty\! i(1\!-\!\delta)^i\!\leq\!\nicefrac{1\!-\!\delta}{\delta^2}$ and $\sum_{i=0}^\infty (1\!-\!\delta)^i\!\!\leq\!\nicefrac{1}{\delta}$.
Hence, the regret becomes $\mathcal{O}(m\sqrt{T})$, which is order-optimal in $T$ \cite[Sec 5.1]{orabona2021modern}, achieving the optimistic premise.

\underline{\textit{Memory-commensurate}}:
Apart from prediction adaptivity, \texttt{OptFTRL-C}'s performance is interpretable with respect to the spectrum of stateless to stateful systems. Consider the stateless case ($\vec{x}_t = \vec 0, \forall t)$. Hence, $c_{t}(\vec{x}_t, \vec{u}_t) = \dtp{\vec \beta_t}{\vec{u_t}}$. In this case,  \texttt{OptFTRL-C} requires predictions only for the next step (as per \eqref{eq:hint-matrix}), The resulting bound becomes $\mathcal{O}((\sumT \Delta^{0}_t)^{\nicefrac{1}{2}}) = \mathcal{O}((\sum_{t} \|\vec{\beta}_t - \tilde{\vec{\beta}}\|)^{\nicefrac{1}{2}})$,  recovering the optimistic bound for \emph{stateless} online learning \cite{mohri-aistats2016}.

On the other hand, consider a system with a general memory $d$. Then, \texttt{OptFTRL-C} uses predictions not only for the next step, but for the next $d+1$ steps $\{\vec \beta_\tau, \vec{\alpha}_t\}_{\tau=t}^{t+d}$, as dictated by \eqref{eq:hint-matrix}. However, the dependence on future predictions' accuracy decays exponentially (recall that $\Delta_t^{(i)}\propto (1-\delta)^i$). For example, when $d=1$ the resulting bound is $\mathcal{O}\big((\sumT \Delta_{t}^{(0)}+\Delta_{t}^{(1)}+\Delta_{t+1}^{(1)})^{\nicefrac{1}{2}}\big)$. I.e., we pay for the error in $d+1$ predictions\footnote{The fact that earlier errors get repeated is due to the compounding effect.}, but with an exponentially decaying rate.

We now present the tools to prove Theorem \ref{thm:main}. Our proof is structured into two primary parts. First, we demonstrate that the regret in linearized NSC is a specific instance within the OCO-M framework, achieved through a particular selection of separable functions (sub-section \ref{subsec:A}, visualized in Fig. \ref{fig:method} (up)). Second, we establish that the regret of OCO-M with these separable functions is in turn a particular case within the Delayed OCO (OCO-D) framework with a specific structure of delay (sub-section \ref{subsec:B}, visualized in Fig. \ref{fig:method} (down)). While the first part is fairly standard in NSC, we do not reduce its resulting OCO-M instance to standard OCO,  but to OCO-D instead. The connection to delayed online learning was indeed identified in \cite{foster20b-log2}, but incorporating predictions (of unknown quality) was not considered.
\subsection{NSC with linearized costs is separable OCO-M.}
\label{subsec:A}
In this subsection, we show how the cost at each time slot $t$ can be approximated by the sum of a finite number of functions of only the past $d$ decisions. Formally:
\begin{lemma}
\label{lem:approimation}
Given $d \geq \frac{1}{\delta} \log(\frac{z}{\delta\epsilon}T)$, $z\doteq w(\kappa_B \kappa_Mp+1)$.
Then, under Assumptions $1 \& 2$, for any $\epsilon>0$:
\begin{align}
    \sumT\bigg| c_t(\vec{x}_t,\vec{u}_t) -  \sum_{i=0}^{d} f_t^{(i)}(M_{t-i})\bigg| \leq  \alpha\epsilon.
    \notag
\end{align}
\begin{proof}Define the counterfactual state $\hat{\vec{x}}_t$ as the state reached starting from $\vec{0}$ and then executing $d$ DAC actions based on policies $M_{t-d}, M_{t-d+1}, \dots, M_{t-1}$. In other words, this is the state reached at $t$ by following the learner's policies but assuming that $\vec{x}_{t-d-1}$ was $\vec 0$. From \eqref{eq:state-nonstat-m}: 
\begin{align}
    \hat{\vec x}_t = \sum _{i=0}^{d-1} A^i \big(BM_{t-i-1}\ \vec {\overline{w}}_{t-i-2} + \vec w_{t-i-1}\big).
    \label{eq:x-hat}
\end{align}
For now, assume that 
$\!\|\hat{\vec x_t}\!-\! \vec{x}_t \|\!\leq\!\frac{\epsilon}{T}$. Then, by Assump. 2:
\begin{align}
    \sumT| c_t(\vec{x}_t,\vec{u}_t) -  c_t(\vec{\hat x}_t,\vec{ u}_t)| \leq \|\vec{\alpha}_t\|\ \|\hat{\vec x_t}\!-\! \vec{x}_t \| = \alpha\epsilon  .
\end{align}
However, $c_t(\vec{\hat x}_t,\vec{u}_t)$ can be written in terms of $f_t^{(i)}(\cdot)$ using again Assump. $2$ but with the definitions in \eqref{eq:state-trans} and \eqref{eq:action-trans} : 
\begin{align}
    &c_t(\vec{\hat x}_t,\vec{ u}_t) = \dtp{\vec\alpha_t}{\sum_{i=0}^{d-1}\vec\psi^{i}_{t}(M_{t-i-1})} +\dtp{\vec \beta_t}{\vec \psi_t(M_t)} \notag
    \\
    &=\sum_{i=0}^{d-1} \dtp{\vec\alpha_t}{\vec\psi^{i}_{t}(M_{t-i-1})} + \dtp{\vec \beta_t}{\vec \psi_t(M_t)} \notag
    \\ 
    &= \sum_{i=1}^{d} \dtp{\vec\alpha_t}{\vec\psi^{i-1}_{t}(M_{t-i})} + \dtp{\vec \beta_t}{\vec \psi_t(M_t)} \notag
    \\
    &= \sum_{i=1}^d f^{(i)}_{t}(M_{t-i}) + \dtp{\vec \beta_t}{\vec \psi_t(M_t)} = \sum_{i=0}^d f^{(i)}_{t}(M_{t-i}). \notag
\end{align}
It remains to show that $\|\vec{x}_t - \hat{\vec x_t} \| \leq \frac{\epsilon}{T}$. From \eqref{eq:state-nonstat-m} and \eqref{eq:x-hat}:
\begin{align}
    &\|\hat{\vec x_t} - \vec{x}_t \| \leq\|\sum_{i=d}^{t-1} A^i \big(BM_{t-i-1} \overline{\vec w}_{t-i-2} + \vec w_{t-i-1}\!\big)\| \notag
    \\
    &\leq \sum_{i=d}^{t-1} (\|A^iBM\|_{op}\| \vec{\overline w}_{t-i-2}\| + \|A^i\|_{op} \|\vec w_{t-i-1}\|)\notag
    \\
    &\stackrel{(a)}{\leq} \sum_{i=d}^{t-1}\|A^i\|_{op} \|BM\| \notag\|\vec {\overline w}_{t-i-2}\| + \|A^i\|_{op} \|\vec w_{t-i-1}\|
    \\
    &\stackrel{(b)}{\leq} \sum_{i=d}^{t-1} (1-\delta)^iw(\kappa_B \kappa_Mp+1)\stackrel{(c)}{\leq} z\int_{i=d}^{\infty}\!\!e^{-\delta i}di = \frac{z}{\delta}e^{-\delta d}, \notag
\end{align}
where $(a)$ follow from the sub-multipilicitive property of $\|\cdot\|_{\text{op}}$, and $\|B\|_\text{op}\! \leq \! \|B\|$, $(b)$ by Assump. $1$, and the bounds on matrix norms, and $(c)$ from $1-x\leq e^{-x}$ and the definition of $z$. Substituting $d$, makes the last term $\frac{\epsilon}{T}$.
\end{proof}
\end{lemma}
Thus, the closeness of $\vec x_t$ and $\vec{\hat x}_t$ is due to the (strong) stability. Lastly, we note the possibility to set $d$ adaptively without knowledge of $T$ using $d = \frac{1}{\delta} \log(\frac{zt}{\delta\epsilon})$.\footnote{This would result in a sum of the form $\sum_t \nicefrac{1}{t}\leq \log (T)$.}

\subsection{separable OCO-M is OCO-D}
Now that we have approximated the cost at $t$ by a separable function with memory, we show in this subsection that the separated functions can be rearranged to represent an equivalent OCO with delayed feedback formulation. 
\label{subsec:B}
\begin{lemma}
\label{lem:equivalence}
Let $f_t(M_0,\dots,M_d)$ be a separable function: $ f_t(M_0,\dots,M_d) = \sum_{i=0}^d f_{t,i}(M_{d-i})$, $f_{t,i}(M): \mathbb{R}^{(d_u\times d_x p)}\mapsto \mathbb R$ . 
Let $\mathcal{A}$ be an online learning algorithm  whose decisions $M_{t+1}$ depend on the history set $\mathbb{H}_t\doteq \cup_{i=0}^{d} \{ f_{\tau,i}(\cdot)\}_{\tau=1}^t$. Define $J_t(M) \doteq \sum_{i=0}^d f_{t+i,i}(M)$, and the history set w.r.t. $J_t(\cdot)$ as $\mathbb{H}^{J}_t$ Then,
\begin{align}
    & \mathbb{H}^{J}_t = \{J_{\tau}(\cdot)\}_{\tau=1}^{t-d}\label{eq:lemma2-a} \text{, and} 
    \\
    & \sum_{t=1}^{T} f_{t}( M_{t-d},\dots, M_t)= \sum_{t=1}^T J_t(M_t). \label{eq:lemma2-b}
\end{align}
\eqref{eq:lemma2-b} means that the accumulated cost, with memory, is equivalent to that of the memoryless functions $J_t(\cdot)$. However, from \eqref{eq:lemma2-a} $\mathcal{A}$ has delayed feedback w.r.t. $J_t(\cdot)$; when deciding $M_{t+1}$, feedback up to only $t-d$ is available.
\end{lemma}
\begin{proof}
the first part is immediate from the definition of $J_\tau(\cdot)$; any $J_\tau(\cdot)$ with $\tau>t-d$ would require a function that is not in the original history set $\mathbb{H}_t$ (not known at $t$).

The second part is mainly index manipulation:
\begin{align}
    &\sum_{t=1}^{T} f_{t}(M_{t-d},\dots, M_t) = 
     \sum_{t=1}^{T} \sum_{i=0}^{d} f_{t,i}( M_{t-i})
    \\
    &= \sum_{i=0}^{d} \sum_{t=1}^{T-i} f_{t+i,i}(M_{t}) 
    = \sum_{t=1}^{T}\sum_{i=0}^{d}f_{t+i,i}(M_{t}) =\sumT J_{t}(M_{t}) \notag
\end{align}
Where the first equality holds by separability, the second by shifting the sum index and using the convention $f_{t<d}(\cdot) \doteq 0$ w.l.o.g,\footnote{The adversarial rounds can always be prefixed with zero cost rounds. Alternatively, redefine regret to start from  $t=d$ as in see \cite[Sec 2.2]{NIPS2015_38913e1d}. } and the third by using $f_{t>T}(\cdot) \doteq 0$.\end{proof}
Now we are ready to prove Theorem \ref{thm:main}: \begin{proof}
Denote with $\pi^\star$ the cost-minimizing policy, and let $M^\star$ be its parametrization. Then, from \eqref{eq:policy-regret}:
\begin{align}
    \mathcal{R}_T &= \sumT c_t(\vec{x}_t,\vec{ u}_t) - c_t(\vec{x}_t(\pi^\star),\vec{ u}_t(\pi^\star))   \notag
    \\
    &\stackrel{(a)}{\leq} \sumT \big( \sum_{i=0}^{d} f_{t}^{(i)}(M_{t-i}) - \sum_{i=0}^{d} f_{t}^{(i)}(M^\star) \big) + 2 \alpha\epsilon \notag
    \\
    &\stackrel{(b)}{=} \sumT f_t(M_{t-d},\dots,M_t) - f_t(M^\star,\dots,M^\star) + 2 \alpha\epsilon
    \\
    &\stackrel{(c)}{=} \underbrace{\sumT F_t(M_t) - F_t(M^\star)}_{\doteq \mathcal{R}_T^F}\ +\ 2 \alpha\epsilon,  \label{eq:to-use-optimdelay}
\end{align}
where $(a)$ follows by Lemma \ref{lem:approimation}, which gives both an upper bound on the learner cost and lower bound on the benchmark's cost, $(b)$ by writing the sum of $f_t^{(i)}(\cdot)$ as a single function with memory, and $(c)$ by Lemma \ref{lem:equivalence}, with $f_{t,i}(M) = f_t^{(i)}(M)$, and hence, $J_t(M) = F_t(M)$. To bound the  delayed feedback regret $\mathcal{R}_T^F$, we use \cite[Thm. 11]{optimdelay21},  which we restate below using the notation of this paper:\footnote{Mapping their notation to ours, we get $\vec{g}_{t-D:t} = G_{t-d:t}$, $\vec{h}_t = H_t$, $\vec g_{t-D:t} - h_t = \Delta_t$, $\vec{a}_{t,F} = 2\kappa_M \Delta_t$, $\vec{b}_{t,F}=\nicefrac{1}{2}\Delta^2_t$, and $\alpha = \kappa_M^2$.}

\begin{figure*}
     \centering
     \begin{subfigure}[b]{0.3\textwidth}
         \centering
         \includegraphics[width=\textwidth]{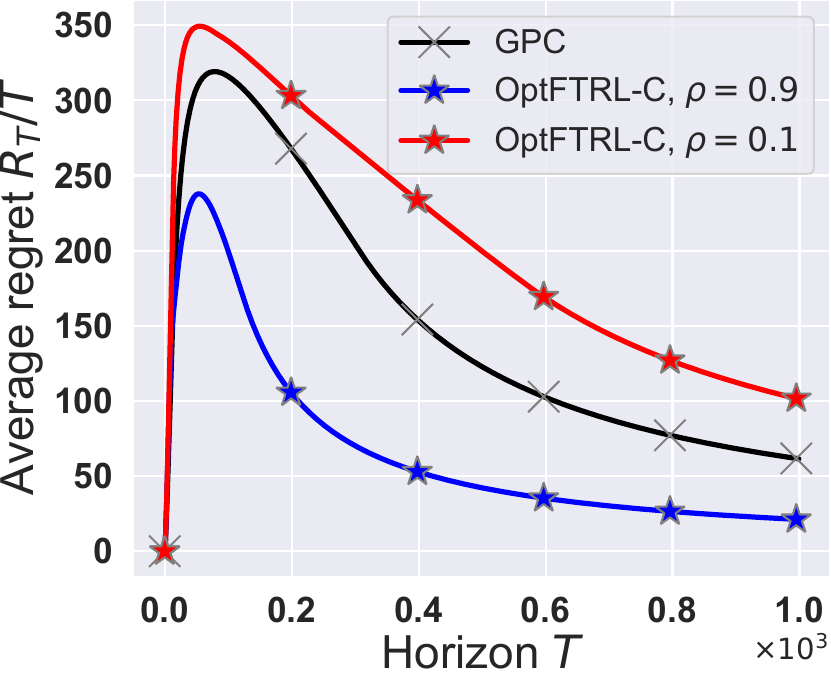}
         \caption{\footnotesize{Stationary costs.}}
         \label{fig:ne_a}
     \end{subfigure}
     \hfill
     \begin{subfigure}[b]{0.31\textwidth}
         \centering
         \includegraphics[width=\textwidth]{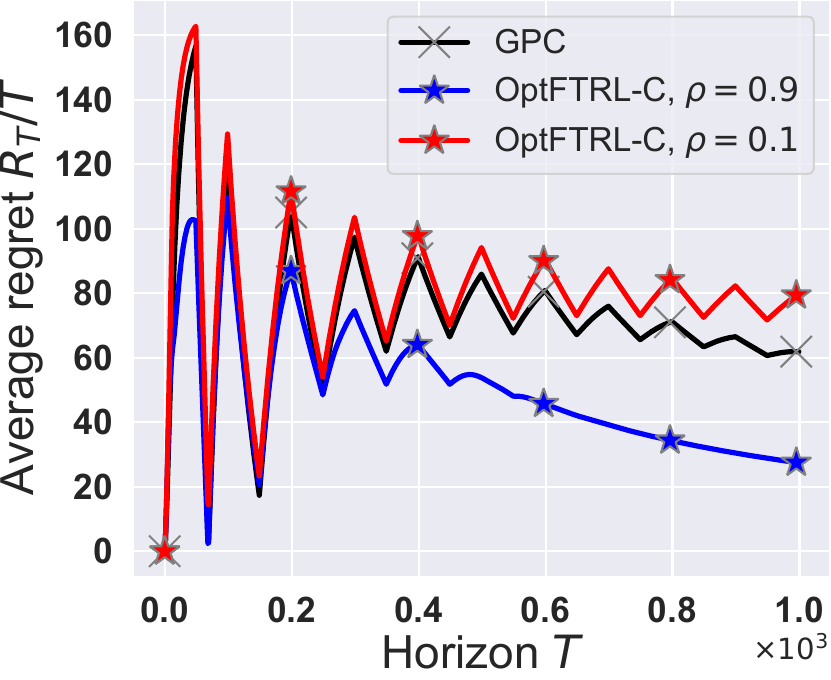}
         \caption{\footnotesize{Alternating costs.}}
         \label{fig:es_b}
     \end{subfigure}
     \hfill
    \begin{subfigure}[b]{0.31\textwidth}
         \includegraphics[width=\textwidth]{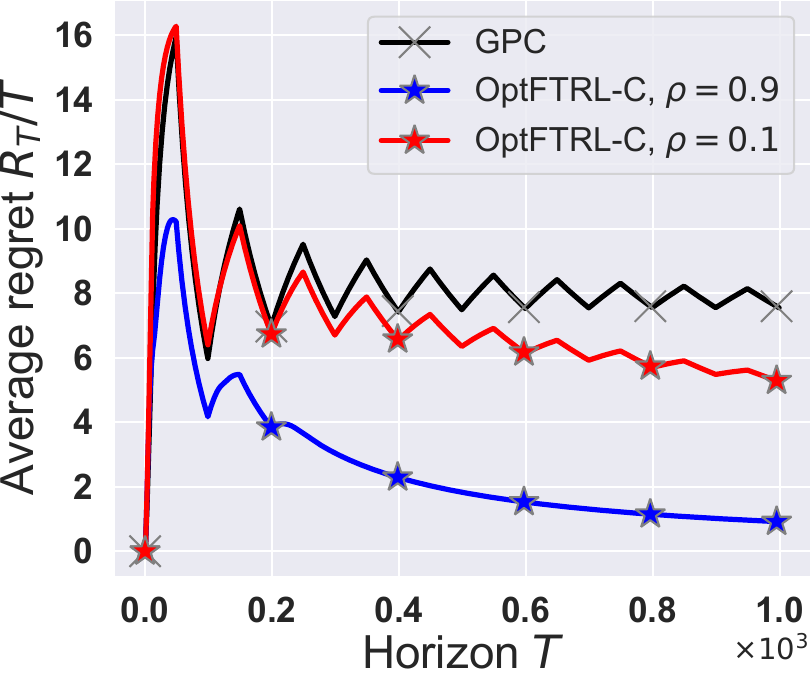}
         \caption{\footnotesize{Alternating costs, low magnitude.}}
         \label{fig:es_c}
    \end{subfigure}
    \caption{The average regret against the optimal policy under various scenarios (cost and disturbances trajectories).}
    \vspace{-0.2cm}
\end{figure*}

\begin{theorem} [\text{\cite[Thm. 11]{optimdelay21}}] 
Let $\lambda_t$ be non-decreasing on $t$ defined as in \eqref{eq:lambda-update}. Let $r_t(\cdot)$ be $\lambda_t$ strongly convex regularizer. Then, $\{M_t\}_t$ computed via \eqref{eq:update_step} ensures that: 
\begin{align}
    \mathcal{R}_T^{F}\leq 8\kappa_M\max_{t\in[T]}\Delta_{t-d:t-1} + 2\sqrt{5}\kappa_M\sqrt{\sumT \Delta_t^2} \label{eq:optim-delay-result}.
\end{align}
\vspace{1mm}
\end{theorem}
To write $\Delta_t$ explicitly, note that  $G_{t-d:t}$ can be written as:
\begin{align}
    & G_{t-d:t} = \sum_{i=0}^{d} \sum_{j=0}^{d} G^{(j)}_{t-d+i+j}
\end{align}
hence we get that $\Delta_t = \| G_{t-d:t} - H_t \| \stackrel{\eqref{eq:hint-matrix}}{\leq} $  
\begin{align}
      &\sum_{i=0}^{d-1} \sum_{j=d-i}^d \|{G}^{(j)}_{t-d+i+j} - {\ti G}^{(j)}_{t-d+i+j} \| + \sum_{j=0}^d \Delta^{(j)}_{t+j}=
      \\
      & \sum_{i=0}^d\sum_{j=d-i}^d \Delta_{t-d+i+j}^{(j)} =\! \sum_{i=0}^d\sum_{j=0}^i \Delta_{t+j}^{(j+d-i)}
       =\! \sum_{i=0}^d\sum_{j=i}^d \Delta_{t+i}^{(d-j+i)}
\end{align}
Substituting back in \eqref{eq:optim-delay-result} gives the result.
\end{proof}


\section{Numerical Example}
\label{sec:ne}
Recall that \texttt{OptFTRL-C} was designed to take advantage of a prediction oracle that forecasts future cost functions with unknown accuracy. Theorem \ref{thm:main} then demonstrated that the average policy regret $\nicefrac{\mathcal{R}_T}{T}$ of \texttt{OptFTRL-C} always converges to $0$, but does so faster for accurate predictions. We therefore plot the policy regret of \texttt{OptFTRL-C} when provided with either accurate or inaccurate predictions. The implementation code of the policies \texttt{OptFTRL-C}, GPC, and the benchmark $\pi^\star$, as well as the code necessary to reproduce all experiments, is available at the repository \cite{nsc_cdc_gh}.

We consider a system with $\vec{x}, \vec{u}\!\!\in\!\!\mathbb{R}^2$, $p = d = 10$,\footnote{The DAC parameter $p$ and the cost's memory $d$ are denoted $h$ and $H$ in \cite{pmlr-v97-agarwal19c}, where they are also set equal. While both are commonly referred to as ``memory", we refer to $d$ also as the delay due to the duality we presented.} and hence $M\!\in\! \mathbb{R}^{2\times20}$. The dynamics are $A=0.9\times\boldsymbol{I}_{2}, B= \boldsymbol{I}_2$, with perturbation $\vec{w}_t\in[-1,1]^2$ of maximum magnitude of $w=\sqrt{2}$. We consider a linear cost $c_t=\dtp{\vec{\alpha}_t}{\vec x_t}$, with $\vec{\alpha}_t\in[-1,1]^2$ and hence $\alpha=\sqrt{2}$. With these choices, we have the upper bound on the gradient $\|G_t\|\leq \nicefrac{\alpha\kappa_Bpw}{0.1}\leq 300$. 

In the accurate prediction case, we set $\tilde c_t(\cdot,\cdot) =  c_t(\cdot,\cdot)$ with probability $\rho=0.9$. Otherwise, we set $\tilde c_t(\cdot,\cdot)$ to be uniformally random ($\vec \alpha_t \in [-1,1]^2$). Hence, $\rho$ represents the probability of correctly predicting $c_t(\cdot,\cdot)$, and we sample it at every slot. For inaccurate prediction, we set $\rho=0.1$. We compare with GPC, which was shown to outperform the classical $\mathcal{H}_2$ and $\mathcal{H}_\infty$ controllers in a wide range of situations.  
\begin{table}[t]
    \begin{tabularx}{\linewidth}{
        |>{\centering}m{0.9cm}| 
        |>{\centering\arraybackslash}X
        |>{\centering\arraybackslash}X 
        |>{\centering\arraybackslash}X 
        |>{\centering\arraybackslash}m{1.2cm}|}\cline{2-5}
     \multicolumn{1}{c||}{}& GPC & \texttt{OptFTRL-C}  $\rho=0.9$ & \texttt{OptFTRL-C}  $\rho=0.1$ & Optimal \\
     \hline\hline
     Scenario (a)& $314,694$ & $355,061$  & $274,965$ &$376,198$ \\
     \hline
     Scenario (b) & $34,994$ & $69,468$  & $16,934$ & $96,869$ \\
     \hline
    Scenario (c) & $642$ & $7,276$  & $2,928$ & $8,191$ \\
     \hline
    \end{tabularx}
        \caption{Accumulated reward (negative cost) $-\sum_{t =1}^T c_t(\vec x_t, \vec u_t)$ of the different policies.}
        \vspace{-0.5cm}
\end{table}
\textit{\underline{In scenario (a)}}, The cost trajectory is set as $\vec\alpha_t\!\!=\!\!(1,1), \vec w_t\!=\!(1,1), \forall t$. 
This  represents a simple case where the cost function does not fluctuate. It can be seen from Fig. \ref{fig:ne_a} that  $\texttt{OptFTRL-C}$ provides the expected acceleration when the prediction is accurate, achieving an average of $60.2\%$ smaller $\nicefrac{\mathcal{R}_T}{T}$ value compared to GPC. At the same time, the average regret still \emph{attenuates} at the same $\mathcal{O}(\nicefrac{1}{\sqrt{T}})$ rate even in the case of inaccurate predictions, but with an average performance degradation of $47.3\%$. 

\textit{\underline{In scenario (b)}}, we deploy \emph{an alternating} cost function. Namely,
$\vec \alpha_t$ alternates between $(1, 1)$ and $(-0.5, -0.5)$ every $50$ steps. The disturbances are still $\vec w_t\!=\!(1, 1), \forall t$. This alternating cost represents an adversarial fluctuation in the cost trajectory, and it is where the non-stochastic framework demonstrates its efficacy. Namely, this fluctuation is enough to violate the guarantees of $\mathcal{H}_2$ controllers, but at the same time, it is small in magnitude, rendering $\mathcal{H}_\infty$ controller overly pessimistic.  It is worth noting that the fluctuation observed in Fig. $b$ is attributed to the update rule of GPC. This update directly modifies the decision variables $M_{t+1}$ based on the observed cost. In contrast, FTRL aggregates both past and future costs to determine $M_{t+1}$. With accurate predictions, this method does not induce much fluctuation as it foresees the upcoming small disturbance. Overall, \texttt{OptFTRL-C} with good prediction achieves an improvement of $32\%$ in $\nicefrac{\mathcal{R}_T}{T}$ value over GPC, while having a $13.4\%$ degradation when fed with the inaccurate oracle.

\textit{\underline{In scenario (c)}} we also deploy an alternating cost function but with different lower magnitudes. Namely, $\vec \alpha_t$ alternates between $(0.1, 0.1)$ and $(-0.5, -0.5)$ every $50$ steps, and $w_t=(0.1, 0.1)$. The goal of this scenario is to show that $\texttt{OptFTRL-C}$ can have an advantage over GPC even regardless of the prediction quality through \emph{adapatablity} to ``easy" environments (i.e., environments with small gradients). In general, GPC performance takes a hit since its learning rate is tuned with the upper bound for the gradient (i.e., $\alpha = w = \sqrt{2}, T=10^3$). The  alternating frequency, on the other hand, contributes to distinguishing more the effect of good predictions. In this scenario \texttt{OptFTRL-C} achieves an improvement of  $16.8\%$ and $69.8\%$ over GPC for $\rho=0.1$ and $\rho=0.9$, respectively. In summary, \texttt{OptFTRL-C} leverages predictions without sacrificing its resilience to their inaccuracy, or to the costs' adversity.

\section{Conclusion}
This paper looked at the NSC problem with the aim of designing DAC controllers that can leverage predictions of unknown quality. The overreaching goal is to have a controller with meaningful policy regret guarantees that are accelerated by good predictions but not void when these predictions fail. By looking at OCO-M from the lens of Delayed OCO, we were able to present the first optimistic DAC controller. Our work paves the way for further research in the ongoing pursuit of data and learning-driven control.
\section{Acknowledgment}
This work was supported by the Dutch National Growth Fund project ``Future Network Services" and by the European Commission through Grant No. 101139270 ``ORIGAMI".
\bibliography{nscref}
\bibliographystyle{IEEEtran}

\end{document}